\documentclass[final,12pt]{colt2023} 
\DeclareMathOperator*{\argmin}{arg\,min}

\title[Generalization Error of Stochastic Mirror Descent]{The Generalization Error of Stochastic Mirror Descent on Over-Parametrized Linear Models}
\usepackage{times}
\usepackage{bbm}

\coltauthor{%
 \Name{Danil Akhtiamov} \Email{dakhtiam@caltech.edu}\\
 \addr 
Department of Computing and Mathematical Sciences \\
California Institute of Technology \\
Pasadena, CA 91125
 \AND
 \Name{Babak Hassibi} \Email{hassibi@caltech.edu}\\
 \addr Department of Electrical Engineering \\
California Institute of Technology \\
Pasadena, CA 91125%
}

\begin{document}

\maketitle

\begin{abstract}%
Despite being highly over-parametrized, and having the ability to fully interpolate the training data, deep networks are known to generalize well to unseen data. It is now understood that part of the reason for this is that the training algorithms used have certain implicit regularization properties that ensure interpolating solutions with ``good" properties are found. This is best understood in linear over-parametrized models where it has been shown that the celebrated stochastic gradient descent (SGD) algorithm finds an interpolating solution that is closest in Euclidean distance to the initial weight vector. Different regularizers, replacing Euclidean distance with Bregman divergence, can be obtained if we replace SGD with stochastic mirror descent (SMD). Empirical observations have shown that in the deep network setting, SMD achieves a generalization performance that is different from that of SGD (and which depends on the choice of SMD's potential function. In an attempt to begin to understand this behavior, we obtain the generalization error of SMD for over-parametrized linear models for a binary classification problem where the two classes are drawn from a Gaussian mixture model. We present simulation results that validate the theory and, in particular, introduce two data models, one for which SMD with an $\ell_2$ regularizer (i.e., SGD) outperforms SMD with an $\ell_1$ regularizer, and one for which the reverse happens.
\end{abstract}


\section{Introduction}

Stochastic gradient descent (SGD), along with its variants, is the workhorse of modern machine learning. Among these variants is stochastic mirror descent (SMD) \cite{nemirovski1983problem}, which differs from SGD in that, instead of updating the weight vector, one updates the gradient of a so-called ``potential" function of the weight vector, along the negative direction of the instantaneous gradient of the loss function. The potential function is what defines a particular instantiation SMD. It is required to be differentiable and strictly convex. When the potential is the squared Euclidean norm we get SGD. 

In deep learning the models are over-parameterized, typically with a number of parameters that is orders of magnitude larger than the size of the training set. In such a setting, there are uncountably many weight vectors that perfectly interpolate the data. And so, it is not clear which will generalize well on unseen data---some may, some may not \cite{zhang2016understanding}. One of the open questions in deep learning is why SGD almost invariably finds solutions that generalize well? In an attempt to understand this question \cite{gunasekar2017implicit} showed that, in over-parameterized linear models, GD finds an interpolating weight vector that minimizes its Euclidean distance from the initial weight vector. In particular, if we initialize with zero (or rather very close to zero in practice) , it finds an interpolating weight vector with minimum 2-norm. This is what is called "implicit regularization" and is what makes the solution obtained by GD different from other interpolating solutions. \cite{gunasekar2017implicit} further showed that mirror descent does the same, except that it finds an interpolating solution that minimizes its Bregman divergence from the initial weight vector. This observation allows one to impose different regularizations on the interpolating weight vector. In \cite{azizan2019stochastic} these results were extended to SGD and SMD, and then informally extended to nonlinear models, such as deep networks, in \cite{azizan2021stochastic}.

In \cite{azizan2021stochastic} it was empirically observed that the generalization error of SMD, for deep networks initialized with the {\em same} weight vector and trained on the {\em same} training set, varied with the choice of potential function. In particular, for a ResNet-18 network (with 11 million weights) trained on the CIFAR-10 dataset for the different potentials $\ell_1$, $\ell_2$, $\ell_3$, and $\ell_{10}$-norms, the generalization error varied quite noticeably. Surprisingly,  the $\ell_1$ regularizer yielded the worst generalization performance and the $\ell_{10}$ regularizer the best. The $\ell_2$ (corresponding to SGD) and $\ell_3$ regularizers straddled a midway generalization performance. 
This paper is concerned with studying the generalization error of SMD for different potential functions. As a first step in this direction, we will look at the problem of binary classification in over-parameterized linear models where the two classes are drawn from a Gaussian mixture model. In this setting, we obtain the generalization error of SMD for general potentials and study them in more detail for the $\ell_2$ (i.e., SGD) and $\ell_1$ cases. We introduce two data models, one for which SGD outperforms SMD with an $\ell_1$ regularizer, and one for which the reverse happens. In both cases, the empirical results well match the theory. 

The hope is that the results obtained here will guide us to the analysis of nonlinear over-parameterized models and, ultimately, to understanding the generalization behavior of deep networks under different training algorithms. 

The remainder of the paper is organized as follows. Section \ref{sec:prelim} gives some preliminary descriptions of the SMD algorithm and its implicit regularization property, introduces the binary classification problem for Gaussian mixtures, reviews the CGMT framework (the main tool used for our analysis), and introduces the two explicit data models that will be studied and analyzed. Section \ref{sec:main} gives general expressions for the generalization error for linear classification of binary Gaussian mixture models which are the main results of the paper. Section \ref{sec:spec} gives explicit expressions for the specific models considered and Section \ref{sec:num} gives numerical results collaborating the theory and showcasing the relative merits of SGD and $\ell_1$-SMD. The paper concludes with Section \ref{sec:conc}.

\section{Preliminaries}
\label{sec:prelim}

In this section, we provide a brief overview of SMD, of the binary classification model we will be studying, and the of Convex Gaussian Min-max Theorem (CGMT) which is fundamental to our analysis. 

\subsection{Stochastic Mirror Descent}
Let $L(w)$ be a separable loss function of some unknown weight (parameter) vector $w \in \mathbb{R}^d$, $L(w) = \sum_{i=1}^n L_i(w)$, where $L_i(w)$ is usually denoted as local loss function. To minimize $L(\cdot)$ over $w$, one can use classical methods such as SGD, which iteratively updates the weight vector estimate along the negative direction of the instantaneous gradient~\cite{robbins1951stochastic}. SMD is a family of optimization algorithms which includes SGD as a special case~\cite{nemirovski1983problem}. SMD uses a strictly convex differentiable potential function $\psi(\cdot)$ such that the weight vector updates are done in the ``mirrored'' domain determined by $\nabla\psi(\cdot)$
\begin{equation}\label{eq:MD}
\nabla\psi(w_i)=\nabla\psi(w_{i-1})-\eta\nabla L_i(w_{i-1}), \quad i\geq 1,
\end{equation}
where $\eta>0$ is the learning rate. Due to strict convexity, $\nabla\psi(\cdot)$ defines an invertible transformation. It is designed to exploit the geometrical structure of the optimization problem with the appropriate choice of potential function. In particular, the update rule in \eqref{eq:MD} can be equivalently written as 
\begin{equation}
    w_i = \mbox{arg}\min_w D_{\psi}(w,w_{i-1}) + \eta w_t^\top \nabla L(w_{i-1}),
\end{equation}
where $D_{\psi}(\cdot,\cdot)$ is the Bregman divergence with respect to $\psi(\cdot)$:
\begin{equation*}
    D_{\psi}(w,w_{i-1}) = \psi(w)-\psi(w_{i-1}) - \nabla\psi(w_{i-1})^\top(w-w_{i-1}).
\end{equation*}
Note that $D_\psi(\cdot,\cdot)$ is non-negative, convex in its first argument and  $D_\psi(w,w') = 0$ iff $w=w'$, due to strict convexity. Due to this construction, different choices of the potential function $\psi(\cdot)$ yield different optimization algorithms, e.g. $\psi(w) = \frac{1}{2}\|w\|_2^2$ gives SGD.

Recently, an array of works has documented and studied the implicit regularization induced by the MD and SMD algorithms used for optimization~\cite{gunasekar2018implicit,azizan2020study,gunasekar2018characterizing,azizan2019characterization,azizan2021stochastic,azizan2022explicit}. These works considered the setting of modern learning problems which are highly overparameterized, \textit{i.e.}, the number of parameters are significantly larger than the number of training data points. In particular, they consider a training set $\mathcal{D} = \{(x_i,y_i): i=1,\dots,n\}$ where $x_i\in\mathbb{R}^d$ are the inputs, and $y_i\in\mathbb{R}$ are the outputs obtained from an underlying distribution. The learning problem is to fit a model $f(x_i,w)$ (linear or nonlinear) that explains the data in $\mathcal{D}$ with some unknown weight vector $w \in \mathbb{R}^d$. In the overparameterized (interpolating) regime, the problem setting often has $d\gg n$, which results in a manifold of (uncountably infinitely many) solutions, $\mathcal{W}$ that interpolate the training data, \textit{i.e.}, $\mathcal{W}=\{w'\in\mathbb{R}^d\ |\ f(x_i,w')=y_i, (x_i,y_i) \in \mathcal{D}\}$. 

Defining a loss function on individual data points $L_i(w) = \ell(y_i - f(x_i,w))$, for some differentiable non-negative function $\ell(\cdot)$ with $\ell(0)=0$, the aforementioned works showed that SMD converges to the solution of
\begin{equation}
    \label{opt_implicit_bregman}
\begin{aligned}
& \underset{w}{\text{min}}
& & D_{\psi}(w,w_{0})\\
& \text{s.t.}
& & y_i = f(x_i,w),\quad i=1,\dots,n ,
\end{aligned}
\end{equation}
for any initialization $w_0$ if $f$ is linear, \textit{i.e.}, $f(x,w) = x^\top w$, while for a nonlinear $f$, SMD converges to a point on $\mathcal{W}$ which is very close to the solution of \eqref{opt_implicit_bregman}. Further, if $w_0 = \argmin_w \psi(w)$, SMD solves the problem in \eqref{opt_implicit_bregman} for $\psi(w)$ instead of $D_{\psi}(w,w_{0})$. This implicit regularization is clearly observed in practice (see \cite{azizan2021stochastic}) where the solutions of SMD display significantly different generalization performance on the unseen data. The relationship between the potential chosen for SMD and the generalization error is unclear. In this work, we take a step towards understanding this relationship by focusing on the simpler case of linear over-parametrized models and considering a particular binary classification problem that we describe below.

\subsection{Binary Classification for a Gaussian Mixture Model}

We consider a binary classification problem with two classes, where for class 1 the feature vector $x\in\mathbb{R}^d$ is drawn at random from $\mathcal{N}(\mu_1, \Sigma_1)$, with $\mu_1 \in\mathbb{R}^d$ the mean and $\Sigma_1\in\mathbb{R}^{d\times d}$ the covariance matrix, and where the label is chosen as $y=1$. Similarly, for class 2 the regressor is drawn from $\mathcal{N}(\mu_2, \Sigma_2)$ and has label $y=-1$. 

We will consider a linear classifier given by a weight vector $w\in\mathbb{R}^d$. In other words for a given feature vector $x$, we will declare that $x$ belongs to class 1 if $x^Tw>0$ and to class 2 if $x^Tw<0$. It is then straightforward to show the following result.
\begin{lemma}\label{lem:gen_error}
Given a weight vector $w$, and assuming the feature vectors are equally likely to be drawn from class 1 or class 2, the corresponding generalization error for the Gaussian mixture model with means $\mu_1$ and $\mu_2$ and covariance matrices $\Sigma_1, \Sigma_2$ is given by
$$E(w) = \frac{1}{2}Q(\frac{\mu_1^Tw}{\sqrt{w^T \Sigma_1 w}}) + \frac{1}{2}Q(-\frac{\mu_2^Tw}{ \sqrt{w^T \Sigma_2 w}})$$
where $Q(\cdot)$ is the integral of the tail of the standard normal distribution. 
\end{lemma}
\begin{proof} See the Appendix.
\end{proof}

Now assume half the training data is drawn from class 1, i.e., 
\[ x_i\sim \mathcal{N}(\mu_1, \Sigma_1),~~~y_i = 1,~~~~~i=1,\ldots \frac{n}{2} \]
and the other half from class 2:
\[ x_i\sim \mathcal{N}(\mu_2, \Sigma_2),~~~y_i = -1,~~~~~i=\frac{n}{2}+1,\ldots n \]
Since we are in the over-parametrized regime, we will assume that $n<d$. In addition, we will assume that both $n, d\rightarrow\infty$. 

Defining the matrix of features and the vector of labels
\[ X = \left[\begin{array}{cccc} x_1 & x_2 & \ldots & x_n \end{array} \right] ~~~,~~~
y = \left[ \begin{array}{c} 1_{\frac{n}{2}} \\ -1_{\frac{n}{2}} \end{array} \right] \]
where $1_{\frac{n}{2}}$ is the all-one vector of length $\frac{n}{2}$, it is easy to see that, if initialized with $w_0 = \argmin_w \psi(w)$, SMD returns the weight vector that solves
\begin{equation}
    \label{opt_implicit}
\begin{aligned}
& \underset{w}{\text{min}}
& & \psi(w)\\
& \text{s.t.}
& & X^Tw = y
\end{aligned}
\end{equation}
In other words, SMD returns a weight vector $w$ that minimizes the potential $\psi(\cdot)$ among all weight vectors that interpolate the training data. 

The goal of this paper is to compute and characterize the generalization error of SMD using different potentials for the linear binary classifier with Gaussian mixture model. As can be seen from Lemma \ref{lem:gen_error}, this requires us to characterize the four quantities
\[ \mu_1^Tw~~,~~\mu_2^Tw~~,~~w^T\Sigma_1w~~,~~w^T\Sigma_2w \]
In fact, in much of the subsequent analysis, we shall assume $\Sigma_1 = \sigma_1^2I$ and $\Sigma_1 = \sigma_1^2I$, which implies we need only characterize the following three quantities 
\[ \mu_1^Tw~~,~~\mu_2^Tw~~,~~\|w\|^2. \]
Since the data model that we are considering is a Gaussian mixture, we shall make use of the Convex Gaussian Min-Max Theorem (CGMT)~\cite{thrampoulidis2015regularized}, which is a tight and extended version of a classical Gaussian comparison inequality~\cite{gordon1985some}. 

\subsection{Convex Gaussian Min-max Theorem (CGMT)}

The CGMT framework has been developed to analyze the properties of the solutions to non-smooth regularized convex optimization problems and has been successfully applied to characterize the precise performance in numerous applications such as $M$-estimators, generalized lasso, massive MIMO, phase retrieval, regularized logistic regression, adversarial training, and max-margin classifiers ~\cite{stojnic2013framework,thrampoulidis2018precise,salehi2019impact,thrampoulidis2015lasso,abbasi2019performance,salehi2018precise,miolane2021distribution,taheri2021fundamental, aubin2020generalization,javanmard2022precise,montanari2019generalization,salehi2020performance}. In this framework, a given challenging optimization problem denoted as the primary optimization $\textbf{(PO)}$ problem, is associated with a simplified auxiliary optimization $\textbf{(AO)}$ problem from which the optimal solution can be tightly inferred. Specifically, the $\textbf{(PO)}$ and $\textbf{(AO)}$ problems are defined as follows:
\begin{align}
\Phi(\mathbf{G})&:=\min _{\mathbf{w} \in \mathcal{S}_{\mathbf{w}}} \max _{\mathbf{u} \in \mathcal{S}_{\mathbf{u}}} \mathbf{u}^{\top} \mathbf{G w}+\psi(\mathbf{w}, \mathbf{u}) \!\! &\!\!{\textbf{(PO)}} \nonumber\\
\phi(\mathbf{g}, \mathbf{h}) & \!:=\!\!\min _{\mathbf{w} \in \mathcal{S}_{\mathbf{w}}} \max _{\mathbf{u} \in \mathcal{S}_{\mathbf{u}}}\|\mathbf{w}\|_2 \mathbf{g}^{\top} \mathbf{u}\!+\!\|\mathbf{u}\|_2 \mathbf{h}^{\top} \mathbf{w}\!+\!\psi(\mathbf{w}, \mathbf{u}) \!\! &\!\!{\textbf{(AO)}} \nonumber
\end{align}
where $\mathbf{G} \in \mathbb{R}^{m \times n}, \mathbf{g} \in \mathbb{R}^m, \mathbf{h} \in \mathbb{R}^n, \mathcal{S}_{\mathbf{w}} \subset \mathbb{R}^n, \mathcal{S}_{\mathbf{u}} \subset \mathbb{R}^m$ and $\psi: \mathbb{R}^n \times \mathbb{R}^m \rightarrow \mathbb{R}$. Denoting any optimal minimizers of $\textbf{(PO)}$ and $\textbf{(AO)}$ as $\mathbf{w}_{\Phi}:=\mathbf{w}_{\Phi}(\mathbf{G})$ and $\mathbf{w}_\phi:=\mathbf{w}_\phi(\mathbf{g}, \mathbf{h})$, respectively, CGMT result states the following.

\begin{theorem}[CGMT~\cite{thrampoulidis2018precise}] \label{thm:cgmt} In $\textbf{(PO)}$ and $\textbf{(AO)}$, let $\mathcal{S}_{\mathbf{w}}, \mathcal{S}_{\mathbf{u}}$ be convex compact sets, $\psi$ be continuous and convex-concave on $\mathcal{S}_{\mathbf{w}} \times \mathcal{S}_{\mathbf{u}}$, and, $\mathbf{G}, \mathbf{g}$ and $\mathbf{h}$ all have entries iid standard normal. Let $\mathcal{S}$ be an arbitrary open subset of $\mathcal{S}_{\mathbf{w}}$ and $\mathcal{S}^c := \mathcal{S}_{\mathbf{w}} \setminus \mathcal{S}$. Denote by $\Phi_{\mathcal{S}^c}(\mathbf{G}) $ and $\phi_{\mathcal{S}^c}(\mathbf{g}, \mathbf{h}) $ the optimal costs of $\textbf{(PO)}$ and $\textbf{(AO)}$ respectively when $\mathbf{w}$ is minimized over  $\mathcal{S}^c$. If there exist constants $\bar{\phi} < \bar{\phi}_{\mathcal{S}^c}$ such that $\phi(\mathbf{g}, \mathbf{h}) \stackrel{p}{\longrightarrow} \bar{\phi}$, and $\phi_{\mathcal{S}^c}(\mathbf{g}, \mathbf{h}) \stackrel{p}{\longrightarrow} \bar{\phi}_{\mathcal{S}^c}$, (converge in probability), then $\lim _{n \rightarrow \infty} \mathbb{P}\left(\mathbf{w}_{\Phi}(\mathbf{G}) \in \mathcal{S}\right)=1$.
\end{theorem}

The probabilities in Theorem~\ref{thm:cgmt} are with respect to the randomness of $\mathbf{G}, \mathbf{g}$, and $\mathbf{h}$. Notice that from the assumptions in the theorem statement, we know that $\mathbf{w}_{\phi}(\mathbf{g}, \mathbf{h}) \in \mathcal{S}$ with probability approaching $1$ for the $\textbf{(AO)}$ problem. However, Theorem~\ref{thm:cgmt} gives a stronger result and concludes the same characterization for the solution of the seemingly different optimization problem $\textbf{(PO)}$. Appropriate choices of $S$ then allow us to conclude that the desired values of $\mu_1^Tw$, $\mu_2^Tw$, $\|w\|^2$ for $\textbf{(PO)}$ concentrate in the same domain as the same values for $\textbf{(AO)}$ as well as that empirical distributions of $w$ coincide for $\textbf{(PO)}$ and $\textbf{(AO)}$ provided that $n$ is big enough ($n \ge 100$ suffices in practice). In our analysis, we use it to characterize the empirical distributions of weights identified by SMD algorithms and to determine the desired generalization errors.


\subsection{Two Explicit Models}

As we shall subsequently see, the performance of SMD for various potentials will highly depend on the parameters $\mu_1$, $\mu_2$, $\Sigma_1$, $\Sigma_2$ of the Gaussian mixture model. In what follows we shall consider the following two explicit models.

\begin{itemize}
    \item {\bf Model 1}: In this model, we will assume that $\mu_1$ is an iid standard normal vector, $\mu_2 = \sqrt{1 - \epsilon^2}\mu_1 + \epsilon v$, where $v$ is another independent iid standard normal vector. This implies that both mean vectors have a length (roughly) equal to $\sqrt{d}$ and a relative angle 
    \[ \theta = \cos^{-1}(\sqrt{1-\epsilon^2}) = \sin^{-1}(\epsilon). \] 
    For simplicity, we will further take $\sigma_1 = \sigma_2 = 1$ and fix $\epsilon = 0.1$. 

    The parameter $\epsilon$ will allow us to control the angle between the two mean vectors and thereby the difficulty in separating the two classes. The difference between the two mean vectors is spread homogeneously across the entries of the vectors. As we shall subsequently see, through the theory and empirical results, this model gives better generalization results on SGD than $\ell_1$-SMD.
    \item {\bf Model 2}: In this model, we take $\mu_{1i} = \mu_{2i}$ and iid standard normal for $i > 1$, and $\mu_{11} = -\mu_{2i} = t = 2$, $\sigma_1 = \sigma_2 = 1$. In other words, the mean vectors of the two classes differ in only a single component. As expected, linear classification for this model is much more conducive to a sparsifying regularizer and both the theory and empirical results will show that, in terms of generalization performance, $\ell_1$-SMD significantly outperforms SGD. 
\end{itemize}

\subsection{Some Useful Lemmas}

The following lemmas will be of use for making calculations specific to models $1$ and $2$.

\begin{lemma}\label{lem:expected_gen}
    Let $X \sim \mathcal{N}(0, \sigma^2)$. Then $\mathbb{E}[(|X|-1)^2 \mathbbm{1}_{|X| > 1} ] = 2(\sigma^2+1)Q(\frac{1}{\sigma}) - \frac{2\sigma}{\sqrt{2\pi}}e^{-\frac{1}{2\sigma^2}} $.
\end{lemma}

\begin{proof} See Appendix.
\end{proof}

\begin{lemma}\label{lem:sqrt}
    The following equality holds for any $x > 0$:
    $$ \sqrt{x} = \min_{\beta > 0} \frac{1}{2\beta} + \frac{\beta x}{2} $$
\end{lemma}

\begin{proof} See Appendix.
\end{proof}

\section{Main Results}
\label{sec:main}

We begin with a theorem that  holds for arbitrary mirror $\psi$. It will later be used and specialized to study $\ell_1$-SMD and SGD. 

\begin{theorem}\label{psi_SMD_CGMT}
The empirical distribution of the parameters identified by SMD with a mirror $\psi$ applied to the Gaussian mixture model with means $\mu_1, \mu_2$ and covariance matrices  $\sigma_1^2I, \sigma_2^2I$ matches the empirical distribution of $\hat{w}$ obtained by solving the following optimization problem for $w$:
$$ \max_{\alpha \ge 0} \min_{w, \beta \ge 0} \max_{\gamma_1, \gamma_2} \psi(w) + \alpha g^Tw + \frac{\alpha}{2\beta} + \frac{\alpha \beta n}{2}\Vert w \Vert_2^2 + \frac{\gamma_1(\mu_1^Tw-1)}{\sigma_1} + \frac{\gamma_2(\mu_2^Tw+1)}{\sigma_2}  - \frac{\gamma_1 ^ 2 + \gamma_2 ^ 2}{\alpha \beta n}$$
where $g\in\mathbb{R}^d$ is a vector of iid standard normal entries. The values of $\|w\|^2$, $\mu_1^Tw$ and $\mu_2^Tw$ inferred from this optimization problem coincide with the same values for parameters found by SMD.
\end{theorem}
\begin{proof} See Appendix.
\end{proof}
The point of the above theorem is that once the distribution of $w$ is identified from the optimization, the quantities $\|w\|^2$, $\mu_1^Tw$, and $\mu_2^Tw$, necessary to obtain the generalization error, can be computed and the histogram for $w$ will recover the weight histogram for SMD. 

The next two theorems arise as applications of Theorem \ref{psi_SMD_CGMT} when we specialize to 
$\psi(w) = \Vert w \Vert_2^2$ and $\psi(w) = \Vert w \Vert_1$ and study the optimization problem in more detail. These more detailed analyses will allow us to determine the generalization errors of SGD and $\ell_1$-SMD, respectively. 

\begin{theorem}\label{SGD_CGMT}
The empirical distribution of the parameters identified by SGD applied to the Gaussian mixture model with means $\mu_1, \mu_2$ and covariance matrices $\sigma_1^2I, \sigma_2^2I$ matches the empirical distribution of $\hat{w}$ given by 
\begin{eqnarray*}
\hat{w} & = & - \frac{\alpha g}{2 + \alpha \beta n} + \frac{\alpha\beta n}{4\Delta}(\frac{\alpha\beta n}{4}(\Vert \mu_2 \Vert_2^2 + \mu_1^T\mu_2) + \sigma_1^2(\frac{\alpha \beta n}{2} + 1))\mu_1 \\
& & - \frac{\alpha\beta n}{4\Delta}(\frac{\alpha\beta n}{4}(\Vert \mu_1 \Vert_2^2 + \mu_1^T\mu_2) + \sigma_2^2(\frac{\alpha \beta n}{2} + 1))\mu_2
\end{eqnarray*}
where $g\in\mathbb{R}^d$ is a vector of iid standard normal entries and $\alpha$ and $\beta$ are defined as solutions to the following two-dimensional scalar optimization problem: 
$$ \max_{\alpha \ge 0} \min_{\beta \ge 0} -\frac{\alpha^2d}{4(1 + \frac{\alpha\beta n}{2})} - \frac{(\alpha\beta n )^2}{16\Delta}(\Vert \mu_1 \Vert ^2 + \Vert \mu_2 \Vert ^ 2 - (\frac{\sigma_1^2}{\sigma_2^2} + \frac{\sigma_2^2}{\sigma_1^2})\mu_2^T\mu_1) - $$ 
$$ - \frac{(\alpha \beta n)^3}{32\Delta}(\Vert \mu_1 \Vert ^2 + \Vert \mu_2 \Vert ^ 2 - (\frac{\sigma_1^2}{\sigma_2^2} + \frac{\sigma_2^2}{\sigma_1^2})\mu_2^T\mu_1 + (\frac{1}{2\sigma_1^2} + \frac{1}{2\sigma_2^2})(\Vert \mu_1 \Vert^2 \Vert \mu_2 \Vert ^ 2 - (\mu_2^T\mu_1)^2))$$
Here $\Delta$ is also a function of $\alpha$ and $\beta$ and is defined as 
\begin{eqnarray*}
\Delta & = & (\frac{\alpha\beta n}{4})^2(4\sigma_1^2\sigma_2^2 + \Vert \mu_1 \Vert^2 \Vert \mu_2 \Vert ^ 2 - (\mu_2^T\mu_1)^2 + 2(\sigma_1^2\Vert \mu_1 \Vert ^2 + \sigma_2^2\Vert \mu_2 \Vert ^ 2 )) \\
& & + \frac{\alpha \beta n}{4}(2\sigma_1^2 + 2\sigma_2^2 + \sigma_1^2\Vert \mu_1 \Vert^2 + \sigma_2^2\Vert \mu_2 \Vert ^ 2) + \sigma_1^2\sigma_2^2
\end{eqnarray*}

The values of $\|w\|^2$, $\mu_1^Tw$ and $\mu_2^Tw$ inferred from this optimization problem coincide with the same values for parameters found by SGD.
\end{theorem}

\begin{proof}
    See Appendix. 
\end{proof}

Note that ${\hat w}$ is simply a non-zero mean Gaussian vector and so $\|w\|^2$, $\mu_1^Tw$, and $\mu_2^Tw$ can be readily computed, thereby allowing the evaluation of the generalization error via Lemma \ref{lem:gen_error}. 

\begin{theorem}\label{1SMD_CGMT}
The empirical distribution of the parameters identified by $\ell_1$-SMD applied to the Gaussian mixture model with means $\mu_1, \mu_2$ and covariance matrices $\sigma_1^2I, \sigma_2^2I$ matches the empirical distribution of $\hat{w}$ given by 
$$\hat{w}_i = -(\alpha \beta n)^{-1}sign(\frac{\gamma_1}{\sigma_1} \mu_{1i} + \frac{\gamma_2}{\sigma_2} \mu_{2i} + \alpha g_i)\max(0, |\frac{\gamma_1}{\sigma_1} \mu_{1i} + \frac{\gamma_2}{\sigma_2}\gamma_2 \mu_{2i} + \alpha g_i| - 1),$$ 
where the $g_i$ are iid standard normal and $\gamma_1, \gamma_2 ,\alpha$ and $\beta$ are defined as solutions of the following four-dimensional optimization problem: 
$$\max_{\alpha \ge 0} \min_{\beta \ge 0} \max_{\gamma_1, \gamma_2} \frac{\gamma_2}{\sigma_2} - \frac{\gamma_1}{\sigma_1}  + \frac{\alpha}{2\beta} - \frac{\gamma_1 ^ 2 + \gamma_2 ^ 2}{\alpha \beta n} - \sum_i \frac{|\alpha g_i  + \frac{\gamma_1}{\sigma_1}\mu_{1i} + \frac{\gamma_2}{\sigma_2}\mu_{2i}| - 1}{2\alpha \beta n} \max(0, |\alpha g_i  + \frac{\gamma_1}{\sigma_1}\mu_{1i} + \frac{\gamma_2}{\sigma_2}\mu_{2i}| - 1) $$
The values of $\|w\|^2$, $\mu_1^Tw$ and $\mu_2^Tw$ inferred from this optimization problem coincide with the same values for parameters found by $\ell_1$- SMD.
\end{theorem}

\begin{proof}
    See Appendix. 
\end{proof}

We should comment that in Theorem \ref{SGD_CGMT} the optimization for the parameters $\alpha$ and $\beta$ is deterministic. However, in Theorem \ref{1SMD_CGMT}, the optimization for $\alpha$, $\beta$, $\gamma_1$, and $\gamma_2$ is stochastic. However, if we make some statistical assumptions on the means $\mu_1$ and $\mu_2$, then by the law of large numbers the term 
\[ \sum_i \frac{|\alpha g_i  + \frac{\gamma_1}{\sigma_1}\mu_{1i} + \frac{\gamma_2}{\sigma_2}\mu_{2i}| - 1}{2\alpha \beta n}\max(0, |\alpha g_i  + \frac{\gamma_1}{\sigma_1}\mu_{1i} + \frac{\gamma_2}{\sigma_2}\mu_{2i}| - 1)\]
will concentrate. In fact, this is why we consider the explicit Models 1 and 1 described earlier. \\

\section{Specific Results}
\label{sec:spec}

We now specialize the previous theorems to the Models 1 and 2 described earlier. This will allow us to get explicit expressions for the generalization error and to compare the performances of SGD and $\ell_1$-SMD.

\subsection{Model 1}

Recall here that $\mu_1$ is standard normal, $\mu_2 = \sqrt{1 - \epsilon^2}\mu_1 + \epsilon v$, where $v$ is an independent standard normal vector. We further assume $\sigma_1 = \sigma_2 = 1$. 

\begin{lemma}\label{expected_model_1}
Denote $\sigma^2 = \gamma_1^2 + \gamma_2 ^ 2 + \alpha ^ 2 + 2\gamma_1\gamma_2\sqrt{1 - \epsilon^2}$. The following equality holds:
$$\mathbb{E}_{\mu_{1i}, v_i, g_i}[(|\gamma_1 \mu_{1i} + \gamma_2 \mu_{2i} + \alpha g_i|-1) \max(0, |\gamma_1 \mu_{1i} + \gamma_2 \mu_{2i} + \alpha g_i|-1) ] = 2(\sigma^2+1)Q(\frac{1}{\sigma}) - \frac{2\sigma}{\sqrt{2\pi}}e^{-\frac{1}{2\sigma^2}}$$

\end{lemma}

\begin{proof}
    Note that $$\gamma_1 \mu_{1i} + \gamma_2 \mu_{2i} + \alpha g_i = (\gamma_1 + \sqrt{1 - \epsilon^2} \gamma_2)\mu_{1i} +  \epsilon \gamma_2 v + \alpha g_i \sim \mathcal{N}(0, \gamma_1^2 + \gamma_2 ^ 2 + \alpha ^ 2 + 2\gamma_1\gamma_2\sqrt{1 - \epsilon^2})$$

     Denote $X = \gamma_1 \mu_{1i} + \gamma_2 \mu_{2i} + \alpha g_i \sim \mathcal{N}(0, \sigma^2)$.

    The initial expectation can then be rewritten in the following way and found using Lemma \ref{lem:expected_gen}:

    $$\mathbb{E}[(|X|-1)^2 \mathbbm{1}_{|X| > 1}] = 2(\sigma^2+1)Q(\frac{1}{\sigma}) - \frac{2\sigma}{\sqrt{2\pi}}e^{-\frac{1}{2\sigma^2}}$$
\end{proof}

\begin{lemma}\label{dot_prods}

Under the terminology from Theorem \ref{psi_SMD_CGMT}, the following equalities hold:

$$\gamma_1 = \frac{\alpha \beta n}{2\sigma_1}(\mu_1^Tw - 1), \gamma_2 = \frac{\alpha \beta n}{2\sigma_2}(\mu_2^Tw + 1)$$

\end{lemma}

\begin{proof}

Follows immediately from taking derivatives by $\gamma_1$ and $\gamma_2$ and equating them to $0$. 
    
\end{proof}



\begin{remark}\label{model1_approx}

Since we work in the asymptotic regime $d \to \infty$, we can replace 
\[ \sum_i \frac{(|\gamma_1 \mu_{1i} + \gamma_2 \mu_{2i} + \alpha g_i|-1)}{2 \alpha \beta n} \max(0, |\gamma_1 \mu_{1i} + \gamma_2 \mu_{2i} + \alpha g_i|-1) \]
from the objective of Theorem \ref{1SMD_CGMT} by $\frac{d(\sigma^2+1)}{\alpha \beta n}Q(\frac{1}{\sigma}) - \frac{d\sigma}{\sqrt{2\pi}\alpha \beta n}e^{-\frac{1}{2\sigma^2}}$. Note that this expression is invariant to the transformation $(\gamma_1, \gamma_2) \to (-\gamma_2, -\gamma_1)$ and so is the rest of the objective from Theorem \ref{1SMD_CGMT}. Since this objective is strictly concave in $\gamma_1$ and $\gamma_2$, we conclude that the optimal parameters must satisfy $\gamma_2 = -\gamma_1$. All this being said, we experiment with the following three-dimensional optimization problem in the numerical part of the work related to model 1:

$$ \max_{\alpha \ge 0} \min_{\beta \ge 0} \max_{\gamma_1} \frac{\alpha}{2\beta}  - \frac{2\gamma_1 ^ 2}{\alpha \beta n} - 2\gamma_1 - \frac{d(\sigma^2+1)}{\alpha \beta n}Q(\frac{1}{\sigma}) + \frac{d\sigma}{\sqrt{2\pi}\alpha \beta n}e^{-\frac{1}{2\sigma^2}}$$
$$\text{where } \sigma^2 = 2\gamma_1^2(1 -\sqrt{1 - \epsilon^2})  + \alpha ^ 2$$

The same lemma suggests us an approximation for $\Vert w \Vert_2^2$, 
since 
\[ w_i^2 = \frac{(|\gamma_1 \mu_{1i} + \gamma_2 \mu_{2i} + \alpha g_i|-1)^2}{(\alpha \beta n)^2}\mathbbm{1_{|\gamma_1 \mu_{1i} + \gamma_2 \mu_{2i} + \alpha g_i| > 1}}$$ we will approximate $\Vert w \Vert_2^2 = \sum_i w_i^2$ as $$ \Vert w \Vert_2^2 \approx  \frac{2d(\sigma^2+1)}{(\alpha \beta n)^2}Q(\frac{1}{\sigma}) - \frac{2d\sigma}{\sqrt{2\pi}(\alpha \beta n)^2}e^{-\frac{1}{2\sigma^2}}. \]
Finally, we find dot products $\mu_1^Tw$ and $\mu_2^Tw$ using Lemma \ref{dot_prods}:
$$\mu_1^Tw = \frac{2\gamma_1}{\alpha\beta n} + 1,  \mu_2^Tw = -\frac{2\gamma_1}{\alpha \beta n} - 1$$

\end{remark}

\subsection{Model 2}

Recall, as before that $\mu_{1i} = \mu_{2i}$ is iid standard normal for $i > 1$ and that $\mu_{11} = -\mu_{2i} = t = 2$. For simplicity, we take $\sigma_1 = \sigma_2 = 1$.

\begin{remark}\label{model_2_approx}

Analogously to Remark \ref{model1_approx}, we will use Lemma \ref{expected_model_1} to approximate the objective by a simpler expression. In this case, the first term of the sum 
\[ \sum_i \frac{(|\gamma_1 \mu_{1i} + \gamma_2 \mu_{2i} + \alpha g_i|-1)}{2 \alpha \beta n} \max(0, |\gamma_1 \mu_{1i} + \gamma_2 \mu_{2i} + \alpha g_i|-1) \]  
cannot be replaced by anything rather than $\frac{(|\gamma_1 t - \gamma_2 t + \alpha g_1|-1)}{2 \alpha \beta n} \max(0, |\gamma_1 t - \gamma_2 t + \alpha g_1|-1)$ itself, so we just leave it this way. Note that for $i > 1$ we have $\gamma_1 \mu_{1i} + \gamma_2 \mu_{2i} + \alpha g_i = (\gamma_1 + \gamma_2) \mu_{2i} + \alpha g_i \sim \mathcal{N}(0, (\gamma_1 + \gamma_2)^2 + \alpha^2)$. Thus, we replace the sum of the other terms by $ \frac{(d-1)(\sigma^2+1)}{\alpha \beta n}Q(\frac{1}{\sigma}) - \frac{(d-1)\sigma}{\sqrt{2\pi}\alpha \beta n}e^{-\frac{1}{2\sigma^2}}$ following the same reasoning as in Remark \ref{model1_approx}, where $\sigma^2 = (\gamma_1 + \gamma_2)^2 + \alpha^2$ this time. Note that this again makes the entire objective invariant to the same transformation $(\gamma_1, \gamma_2) \to (-\gamma_2, -\gamma_1)$. We conclude that $\gamma_2 = -\gamma_1$ and $\sigma = \alpha$, which leads us to:

$$ \max_{\alpha \ge 0} \min_{\beta \ge 0} \max_{\gamma_1} \frac{\alpha}{2\beta}  - \frac{2\gamma_1 ^ 2}{\alpha \beta n} - 2\gamma_1 - \frac{(|2\gamma_1 t + \alpha g_1|-1)}{2 \alpha \beta n} \max(0, |2\gamma_1 t + \alpha g_1|-1) - $$
$$ - \frac{(d-1)(\alpha^2+1)}{\alpha \beta n}Q(\frac{1}{\alpha}) + \frac{(d-1)}{\sqrt{2\pi} \beta n}e^{-\frac{1}{2\alpha^2}}$$

Again analogously to Remark \ref{model1_approx}, we obtain:
$$\Vert w \Vert_2^2 \approx  \frac{(|2\gamma_1 t + \alpha g_1|-1)^2}{(\alpha \beta n)^2}\mathbbm{1}_{|2\gamma_1 t + \alpha g_1| > 1} + \frac{2(d-1)(\alpha^2+1)}{(\alpha \beta n)^2}Q(\frac{1}{\alpha}) - \frac{2(d-1)}{\sqrt{2\pi}\alpha( \beta n)^2}e^{-\frac{1}{2\alpha^2}}$$
$$\mu_1^Tw = \frac{2\gamma_1}{\alpha\beta n} + 1,  \mu_2^Tw = -\frac{2\gamma_1}{\alpha \beta n} - 1$$

\end{remark}

\section{Numerical simulations}
\label{sec:num}

This section provides a comparison between classification errors obtained by training linear models using SGD and $\ell_1$-SMD and evaluating the corresponding performances empirically to the classification errors predicted by CGMT. We used code from a publically available repository \url{https://github.com/SahinLale/StochasticMirrorDescent} provided by authors of \cite{azizan2021stochastic} with minor changes for training. CGMT predictions were calculated numerically by solving the corresponding optimization problems via a grid search and then using Remarks \ref{model1_approx} and \ref{model_2_approx} along with Lemma \ref{lem:gen_error} to evaluate the error. In the tables presented below, CGMT $\ell_1$-SMD stands for the classification error predicted by CGMT for $\ell_1$ stochastic mirror descent,  empirical $\ell_1$ stands for the test error evaluated for a trained $\ell_1$-SMD initialized near $0$, CGMT SGD and empirical SGD signify the same values, but for SGD. The prediction of $\ell_1$-SMD and the empirical results for $\ell_1$-SMD appear to not depend too dramatically on the realizations of $\mu_1$ and $\mu_2$. That is, they seem to be well-concentrated for model $1$. However, they were less so for model $2$, so we averaged both over $5$ evaluations each. The prediction of CGMT SGD and the empirical performance of SGD were observed to be well-concentrated for both models.  As the reader can see, the match between the empirical and CGMT-predicted SGD performances is better than between the same quantities for $\ell_1$-SMD. We believe that this arises because the latter is more challenging numerically, as the corresponding expression for CGMT involves a $3$ - dimensional optimization instead of $2$ -dimensional and is more sensitive to parameter changes. Apart from that, evaluating the performance of $\ell_1$-SMD empirically is also more challenging because it requires more iterations to converge. That is, there always is a chance that the generalization errors could match more closely if the algorithm was run for more iterations. In either event, the match between the theoretical and empirical generalization errors is quite good in all cases. 

As can be seen from the Tables, for Model 1, SGD has slightly superior performance compared to $\ell_1$-SMD. This is reasonable, since the difference between the two mean vectors is spread homogeneously across the entries of the vectors. 

However, for Model 2, $\ell_1$-SMD has significantly better performance. Again, this is expected because the mean vectors of the two classes differ in only a single component. Therefore linear classification for this model is much more conducive to a sparsifying regularizer. 

These results clearly demonstrate that the generalization performance of linear classifiers on binary Gaussian mixture models tangibly depends on the mirror used by the training algorithm and on the model the data obeys. We believe this general principle to hold for deep networks as well, although it will merit a much more difficult and detailed analysis.

\subsection{Model 1}

\begin{center}
\begin{tabular}{||c c c c c c||} 
 \hline
 n & d & CGMT $\ell_1$ & Empirical $\ell_1$ & CGMT SGD & Empirical SGD \\ [0.5ex] 
 \hline\hline
 500 & 1000 & 0.242 & 0.275 & 0.202 & 0.191\\
 \hline
 200 & 1000 & 0.315 & 0.306 & 0.199 & 0.194\\  
 \hline 
 100 & 1000 & 0.370 & 0.346  & 0.253 & 0.249\\ 
 \hline

 1000 & 10000 & 0.023  & 0.012 & 0 & 0 \\ [1ex] 
 
 \hline

\end{tabular}
\end{center}

\subsection{Model 2}

\begin{center}
\begin{tabular}{||c c c c c c||} 
 \hline
 n & d & CGMT $\ell_1$ & Empirical $\ell_1$ & CGMT SGD & Empirical SGD \\ [0.5ex] 
 \hline\hline
 100 & 1000 & 0.056 & 0.059  & 0.155 & 0.152\\ 
 \hline
1000 & 10000 & 0.045 & 0.051 & 0.152 & 0.150\\
 \hline

 500 & 10000 & 0.048  & 0.034 & 0.211 & 0.218 \\ [1ex] 
 
 \hline

\end{tabular}
\end{center}

\section{Conclusion}
\label{sec:conc}

In this paper we studied the problem of linear classification of binary Gaussian mixture models using SMD training with general potentials. Using a CGMT analysis we are able to find expressions for the generalization error. Numerical simulations show a good agreement between the theory and empirical results. In particular, we observe that the generalization performance depends heavily on the mirror used in SMD, as well as on the data model. We exhibited two models, one for which SGD was superior and one for which $\ell_1$-SMD is so. There are several directions in which these results can be extended. One is to find explicit expressions for the generalization performance of other mirrors, most notably $\ell_\infty$. Another is to extend the classification problem beyond the binary case to more complicated Gaussian mixtures. Finally, we consider the work performed here to be a small step in the direction of understanding the generalization performance of deep networks.

\acks{We are grateful to Sahin Lale for helping us with numerical experiments.}

\bibliography{main}

\appendix


\section{Technical proofs}

\begin{proof}{\bf of Lemma \ref{lem:gen_error}}
    By definition, 
    
    $$E(w) = \frac{1}{2}\mathbb{P}_{x \sim \mathcal{N}(\mu_1, \Sigma_1})(w^Tx < 0) + \frac{1}{2}\mathbb{P}_{x \sim \mathcal{N}(\mu_2, \Sigma_2)}(w^Tx \ge 0)$$

    Rewrite $x = \mu_1 + y_1$ for $x \sim \mathcal{N}(\mu_1, \Sigma_1)$ and $x = \mu_2 + y_2$ for $x \sim \mathcal{N}(\mu_2, \Sigma_2)$. Note that $y_1 \sim \mathcal{N}(0, \Sigma_1)$ and $y_2 \sim \mathcal{N}(0, \Sigma_2)$. We obtain:

    $$E(w) = \frac{1}{2}\mathbb{P}_{y_1 \sim \mathcal{N}(0, \Sigma_1)}(w^Ty_1 < -\mu_1^Tw) + \frac{1}{2}\mathbb{P}_{y_2 \sim \mathcal{N}(0, \Sigma_2)}(w^Ty_2 \ge - \mu_2^Tw)$$

    Since $z_1 = w^Ty_1 \sim \mathcal{N}(0, w^T\Sigma_1w)$ and $z_2 = w^Ty_2 \sim \mathcal{N}(0, w^T\Sigma_2w)$ we have:

    $$E(w) = \frac{1}{2}\mathbb{P}_{z_1 \sim \mathcal{N}(0, w^T\Sigma_1w)}(z_1 < -\mu_1^Tw) + \frac{1}{2}\mathbb{P}_{z_2 \sim \mathcal{N}(0, w^T\Sigma_2w)}(z_2 \ge - \mu_2^Tw) = $$
    $$ = \frac{1}{2}\mathbb{P}_{z'_1 \sim \mathcal{N}(0, 1)}(z'_1 < - \frac{\mu_1^Tw}{\sqrt{w^T \Sigma_1 w}}) + \frac{1}{2}\mathbb{P}_{z'_2 \sim \mathcal{N}(0, 1)}(z'_2 \ge -\frac{\mu_2^Tw}{ \sqrt{w^T \Sigma_2 w}}) = $$
    $$= \frac{1}{2}\mathbb{P}_{z'_1 \sim \mathcal{N}(0, 1)}(z'_1 > \frac{\mu_1^Tw}{\sqrt{w^T \Sigma_1 w}}) + \frac{1}{2}\mathbb{P}_{z'_2 \sim \mathcal{N}(0, 1)}(z'_2 \ge -\frac{\mu_2^Tw}{ \sqrt{w^T \Sigma_2 w}}) = $$
    $$= \frac{1}{2}Q(\frac{\mu_1^Tw}{\sqrt{w^T \Sigma_1 w}}) + \frac{1}{2}Q(-\frac{\mu_2^Tw}{ \sqrt{w^T \Sigma_2 w}})$$
\end{proof}

\begin{proof}{\bf of Lemma \ref{lem:expected_gen}} Denote $Y = \frac{X}{\sigma} \sim \mathcal{N}(0,1)$.

$$\mathbb{E}[(|X|-1)^2 \mathbbm{1}_{|X| > 1}] = \mathbb{E}[X^2 \mathbbm{1}_{|X| > 1}] - 2\mathbb{E}[|X| \mathbbm{1}_{|X| > 1}] + \mathbb{E}[\mathbbm{1}_{|X| > 1}] = 2\mathbb{E}[X^2 \mathbbm{1}_{X > 1}] - 4\mathbb{E}[X \mathbbm{1}_{X > 1}] + 2\mathbb{E}[\mathbbm{1}_{X > 1}] = $$
$$ = 2\sigma^2\mathbb{E}[Y^2 \mathbbm{1}_{Y > \frac{1}{\sigma}}] - 4 \sigma\mathbb{E}[Y\mathbbm{1}_{Y > \frac{1}{\sigma}}] + 2\mathbb{E}[\mathbbm{1}_{Y > \frac{1}{\sigma}}] = \frac{2\sigma^2}{\sqrt{2\pi}} \int^{+\infty}_{\frac{1}{\sigma}}t^2 e^{-\frac{t^2}{2}} dt - \frac{4\sigma}{\sqrt{2\pi}}\int^{+ \infty}_{\frac{1}{\sigma}} t e^{-\frac{t^2}{2}} dt + 2Q(\frac{1}{\sigma}) = $$
$$ = -\frac{2\sigma^2}{\sqrt{2\pi}} \int^{+\infty}_{\frac{1}{\sigma}}t d e^{-\frac{t^2}{2}}  - \frac{4\sigma}{\sqrt{2\pi}}\int^{+ \infty}_{\frac{1}{\sigma}} e^{-\frac{t^2}{2}} d\frac{t^2}{2} + 2Q(\frac{1}{\sigma}) = - \frac{2\sigma^2}{\sqrt{2\pi}} e^{-\frac{t^2}{2}}t _{|_{\frac{1}{\sigma}}^{+\infty}}  + \frac{2\sigma^2}{\sqrt{2\pi}} \int^{+\infty}_{\frac{1}{\sigma}} e^{-\frac{t^2}{2}} dt - $$   
$$- \frac{4\sigma}{\sqrt{2\pi}}\int^{+ \infty}_{\frac{1}{2\sigma^2}} e^{-z} dz + 2Q(\frac{1}{\sigma}) = \frac{2\sigma}{\sqrt{2\pi}}e^{-\frac{1}{2\sigma^2}} + 2\sigma^2Q(\frac{1}{\sigma}) - \frac{4\sigma}{\sqrt{2\pi}}e^{-\frac{1}{2\sigma^2}} + 2Q(\frac{1}{\sigma}) =$$
$$ = 2(\sigma^2+1)Q(\frac{1}{\sigma}) - \frac{2\sigma}{\sqrt{2\pi}}e^{-\frac{1}{2\sigma^2}}  $$
    
\end{proof}

\begin{proof}{\bf of Lemma \ref{lem:sqrt}}
    Differentiate the objective from the right hand side by $\beta$:

    $$ \frac{d}{d\beta}(\frac{1}{2\beta} + \frac{\beta x}{2}) = \frac{1}{-2\beta^2} + \frac{x}{2} $$

    We conclude that $\frac{1}{2\beta} + \frac{\beta x}{2}$ is minimized at $\beta = \frac{1}{\sqrt{x}}$. The value the objective takes at this point is 
    
$\frac{1}{2\beta} + \frac{\beta x}{2} = \frac{1}{\frac{2}{\sqrt{x}}} + \frac{\frac{1}{\sqrt{x}}x}{2} = \sqrt{x}$
\end{proof}

\begin{proof}{\bf of Theorem \ref{psi_SMD_CGMT}}

Using equation \eqref{opt_implicit} it is straightforward to see that SMD with mirror $\psi$ converges to $\hat{w}$ solving the following optimization problem for $w$: 
$$ \min_w \max_{\lambda} \psi(w) + \lambda^T (X^T w - y)$$

Denote by $M$ the $d \times n$ matrix satisfying $M_{ij} = \mu_{1i}$ if $i \le n / 2$ and $M_{ij} = \mu_{2i}$ otherwise. In words, $M$ is the matrix whose first $n / 2$ columns are $\mu_1$ and whose last $n / 2$ columns are $\mu_2$. Denote the random matrix with independent standard Gaussian entries by $\tilde{X}$. Also denote the vectors consisting of the first $n/2$ and last $n/2$ coordinates of $\lambda$ by $\lambda_1$ and $\lambda_2$ respectively so that $\lambda = \begin{pmatrix} \lambda_1 \\ \lambda_2 \end{pmatrix}$. Finally, define $\tilde{\lambda_i} = \sigma_i\lambda_i, i = 0, 1$ and $\tilde{\lambda} = \begin{pmatrix} \tilde{\lambda_1} \\ \tilde{\lambda_2} \end{pmatrix}$.  We then have 
$$\lambda^TX^Tw - \lambda^Ty= \tilde{\lambda}^T\tilde{X}^Tw + \lambda^TM^Tw - \lambda^Ty = \tilde{\lambda}^T\tilde{X}^Tw + \begin{pmatrix} \lambda_1 \\
\lambda_2 \end{pmatrix}^T \begin{pmatrix}
\mu_1^Tw \mathbbm{1}_{\frac{n}{2}}^T && \mu_2^Tw \mathbbm{1}_{\frac{n}{2}}^T 
\end{pmatrix}^T -\lambda^Ty = $$
$$ = \tilde{\lambda}^T\tilde{X}^T w+ \tilde{\lambda}^T \begin{pmatrix}
\frac{\mu_1^Tw}{\sigma_1} \mathbbm{1}_{\frac{n}{2}}^T && \frac{\mu_2^Tw}{\sigma_2} \mathbbm{1}_{\frac{n}{2}}^T \end{pmatrix}^T - \tilde{\lambda}^T(\frac{\mathbbm{1}_{\frac{n}{2}}^T}{\sigma_1}, \frac{\mathbbm{1}_{\frac{n}{2}}^T}{\sigma_2})$$

Plugging it in in the optimization problem above and denoting $\tilde{\mu_i} = \frac{\mu_i}{\sigma_i}, i = 1, 2$, $m = (\tilde{\mu_1}^Tw-\sigma_1^{-1}, \dots, \tilde{\mu_1}^Tw-\sigma_1^{-1}, \tilde{\mu_2}^Tw-\sigma_2^{-1}, \dots, \tilde{\mu_2}^Tw-\sigma_2^{-1})$ , we obtain:

$$\min_w \max_{\lambda} \psi(w) + \lambda^T (X^T w - y) =  \min_w \max_{\tilde{\lambda}} \tilde{\lambda}^T \tilde{X} w + \psi(w) + \tilde{\lambda}^T m$$ 

Since $\phi(w,\tilde{\lambda}) = \psi(w) + \tilde{\lambda}^T m$ is convex in $w$ and is concave (linear) in $\tilde{\lambda}$ and $\tilde{X}$ is standard normal, we can replace this PO problem by the corresponding AO, which is known to yield solutions with the same empirical distribution and the same distribution of 
$\mu_1^Tw, \mu_2^Tw, \Vert w \Vert_2^2$ according to Theorem \ref{thm:cgmt}:

$$ \min_w \max_{\tilde{\lambda}} \Vert \tilde{\lambda} \Vert g^Tw + \tilde{\lambda}^Th \Vert w \Vert_2 + \psi(w) + \tilde{\lambda}^T m $$

Write $\tilde{\lambda} = \alpha u$, where $\alpha = | \tilde{\lambda} | \ge 0$ and $u$ is unit. We then have: 

$$\min_w \max_{\alpha \ge 0, u} \alpha g^Tw + \tilde{\lambda}^T (h \Vert w \Vert_2 + m) + \psi(w)$$

Since this is clearly maximized when $u$ is aligned with $h \Vert w \Vert_2 + m$, we simplify the expression:

$$\min_w \max_{\alpha \ge 0} \alpha g^Tw + \alpha \Vert h \Vert w \Vert_2 + m \Vert_2 + \psi(w)$$

We will simplify $\Vert h \Vert w \Vert_2 + m \Vert_2$ before proceeding further. Note that 

$$\Vert h \Vert w \Vert_2 + m \Vert_2 = \sqrt{h^Th \Vert w \Vert_2^2 + 2h^Tm\Vert w \Vert_2 + m^Tm}$$

Recall that $h$ is standard normal. Thus, $h^Th$ is almost equal to $n$, because we work in the asymptotic regime. The second term $2h^Tm\Vert w \Vert_2$ is negligible compared to the first for almost any $h$ because $m$ is always in the span of two vectors $(1 , \dots, 1, 0, \dots, 0)^T$ and $( 0, \dots, 0, 1 , \dots, 1)^T$ and the projection of $w$ onto this $2$- dimensional span is negligible for almost any $h$. Finally, $m^Tm = \frac{n}{2\sigma_1^2}(\mu_1^Tw-1)^2 + \frac{n}{2\sigma_2^2}(\mu_2^Tw+1)^2$. Incorporating all these observations into the objective and switching the order of optimization using the convex-concativity of the terms we get:

$$ \max_{\alpha \ge 0} \min_w \alpha g^Tw + \alpha \sqrt{n\Vert w \Vert_2^2 + \frac{n}{2\sigma_1^2}(\mu_1^Tw-1)^2 + \frac{n}{2\sigma_2^2}(\mu_2^Tw+1)^2} + \psi(w)$$

To get rid of the square root, we use Lemma \ref{lem:sqrt} and arrive to:

$$\max_{\alpha \ge 0} \min_{w, \beta \ge 0} \alpha g^Tw + \frac{\alpha}{2\beta} + \frac{\alpha \beta n}{2}(\Vert w \Vert_2^2 + \frac{1}{2\sigma_1^2}(\mu_1^Tw-1)^2 + \frac{1}{2\sigma_2^2}(\mu_2^Tw+1)^2) + \psi(w)$$

Substituting $\frac{\mu_1^Tw-1}{\sigma_1}$ and $\frac{\mu_2w+1}{\sigma_2}$ by $a_1$ and $a_2$ respectively and adding two more scalar variables $\gamma_1, \gamma_2$ we deduce:

$$\max_{\alpha \ge 0} \min_{w, \beta \ge 0} \max_{\gamma_1, \gamma_2, a_1, a_2} \alpha g^Tw + \frac{\alpha}{2\beta} + \frac{\alpha \beta n}{2}(\Vert w \Vert_2^2 + \frac{1}{2}a_1^2 + \frac{1}{2}a_2^2) + \gamma_1(\frac{\mu_1^Tw-1}{\sigma_1} - a_1) + \gamma_2(\frac{\mu_2w+1}{\sigma_2} - a_2) + \psi(w)$$

Taking the derivatives by $a_i, i = 1,2$ and equating them to $0$ leads to $a_i = \frac{2 \gamma_i}{ \alpha \beta n}, i = 1, 2$. Plugging these in and simplifying and regrouping the terms we obtain the desired optimization problem:

$$ \max_{\alpha \ge 0} \min_{w, \beta \ge 0} \max_{\gamma_1, \gamma_2} \psi(w) + \alpha g^Tw + \frac{\alpha}{2\beta} + \frac{\alpha \beta n}{2}\Vert w \Vert_2^2 + \frac{\gamma_1(\mu_1^Tw-1)}{\sigma_1} + \frac{\gamma_2(\mu_2^Tw+1)}{\sigma_2}  - \frac{\gamma_1 ^ 2 + \gamma_2 ^ 2}{\alpha \beta n}$$

\end{proof}

\begin{proof}{\bf of Theorem \ref{SGD_CGMT}}

Put $\psi(w) = \Vert w \Vert^2_2$ in the objective of Theorem $\ref{psi_SMD_CGMT}$:

$$\max_{\alpha \ge 0} \min_{\beta \ge 0} \max_{\gamma_1, \gamma_2} \min_{w} \Vert w \Vert^2_2 + \alpha g^Tw + \frac{\alpha}{2\beta} + \frac{\alpha \beta n}{2}\Vert w \Vert_2^2 + \frac{\gamma_1(\mu_1^Tw-1)}{\sigma_1} + \frac{\gamma_2(\mu_2^Tw+1)}{\sigma_2}  - \frac{\gamma_1 ^ 2 + \gamma_2 ^ 2}{\alpha \beta n}$$

Denote $\tilde{\gamma_i} = \frac{\gamma_i}{\sigma_i}, i = 1,2$ and rewrite it in the following way:

$$\max_{\alpha \ge 0} \min_{\beta \ge 0} \max_{\tilde{\gamma_1}, \tilde{\gamma_2}} \min_{w} \tilde{\gamma_2}- \tilde{\gamma_1}  + \frac{\alpha}{2\beta} - \frac{\sigma_1^2\tilde{\gamma_1} ^ 2 + \sigma_2^2\tilde{\gamma_2} ^ 2}{\alpha \beta n} + \sum_i (1 + \frac{\alpha \beta n}{2}) w_i^2 + w_i (\alpha g_i + \tilde{\gamma_1}\mu_{1i} + \tilde{\gamma_2}\mu_{2i})$$

Thus, the minimization over $w$ reduces to minimization over $w_i$ for each $w_i$ separately. The latter is straightforward because the objective of the minimization is just a quadratic polynomial. Therefore,
the optimal $w_i = - \frac{\alpha g_i + \tilde{\gamma_1} \mu_{1i} + \tilde{\gamma_2} \mu_{2i}}{2 + \alpha \beta n}$ for each $i$ and thus $w = - \frac{\alpha g + \tilde{\gamma_1} \mu_{1} + \tilde{\gamma_2} \mu_{2}}{2 + \alpha \beta n}$. Hence, we obtain the following optimization problem:

$$\max_{\alpha \ge 0} \min_{\beta \ge 0} \max_{\tilde{\gamma_1}, \tilde{\gamma_2}} \tilde{\gamma_2} - \tilde{\gamma_1}  + \frac{\alpha}{2\beta} - \frac{\sigma_1^2\tilde{\gamma_1} ^ 2 + \sigma_2^2\tilde{\gamma_2} ^ 2}{\alpha \beta n} - \frac{1}{4 + 2\alpha \beta n}\sum_i (\alpha g_i + \tilde{\gamma_1} \mu_{1i} + \tilde{\gamma_2} \mu_{2i})^2$$

We will simplify the sum before proceeding further with the expression. First, note that

$$\sum_i (\alpha g_i + \tilde{\gamma_1} \mu_{1} + \tilde{\gamma_2} \mu_{2})^2 = \sum_{i} \alpha^2 g_i^2 + \tilde{\gamma_1}^2 \mu_{1i}^2 + \tilde{\gamma_2}^2 \mu_{2i}^2 + 2\alpha(\tilde{\gamma_1} \mu_{1i}g_i + \tilde{\gamma_2} \mu_{2i}g_i)+2\tilde{\gamma_1}\tilde{\gamma_2}\mu_{1i}\mu_{2i} = $$
$$ = \alpha^2 \Vert g \Vert_2^2 + \tilde{\gamma_1}^2 \Vert \mu_1 \Vert_2^2 + \tilde{\gamma_2}^2 \Vert \mu_2 \Vert_2^2 + 2 \alpha \tilde{\gamma_1} \mu_1^Tg + 2 \alpha \tilde{\gamma_2} \mu_2^Tg + 2 \tilde{\gamma_1} \tilde{\gamma_2} \mu_1^T\mu_2$$

Since $g$ is standard normal and $\mu_1, \mu_2$ are two fixed vectors we can ignore the $2 \alpha \tilde{\gamma_1} \mu_1^Tg$ and $2 \alpha \tilde{\gamma_2} \mu_2^Tg$ terms and replace $\Vert g \Vert_2^2$ by $d$ asymptotically. Hence, we can replace the sum with:

$$\alpha^2 d + \tilde{\gamma_1}^2 \Vert \mu_1 \Vert_2^2 + \tilde{\gamma_2}^2 \Vert \mu_2 \Vert_2^2 + 2 \tilde{\gamma_1} \tilde{\gamma_2} \mu_1^T\mu_2$$

Plugging it back into the main objective we have:

$$\max_{\alpha \ge 0} \min_{\beta \ge 0} \max_{\tilde{\gamma_1}, \tilde{\gamma_2}}  \tilde{\gamma_2} - \tilde{\gamma_1}  + \frac{\alpha}{2\beta} - \frac{\sigma_1^2\tilde{\gamma_1} ^ 2 + \sigma_2^2\tilde{\gamma_2} ^ 2}{\alpha \beta n} - \frac{1}{4 + 2 \alpha \beta n}(\alpha^2 d + \tilde{\gamma_1}^2 \Vert \mu_1 \Vert_2^2 + \tilde{\gamma_2}^2 \Vert \mu_2 \Vert_2^2 + 2 \tilde{\gamma_1} \tilde{\gamma_2} \mu_1^T\mu_2)$$

Take the derivatives by $\tilde{\gamma_1}, \tilde{\gamma_2}$ and equate them to zero:

$$-1 - \frac{2 \sigma_1^2\tilde{\gamma_1}}{\alpha \beta n} - \frac{\Vert \mu_1 \Vert_2^2 \tilde{\gamma_1} + \mu_1^T \mu_2 \tilde{\gamma_2}}{2 + \alpha \beta n} = 0 $$
$$1 - \frac{2 \sigma_2^2\tilde{\gamma_2}}{\alpha \beta n} - \frac{\Vert \mu_2 \Vert_2^2 \tilde{\gamma_2} + \mu_1^T \mu_2 \tilde{\gamma_1}}{2 + \alpha \beta n} = 0 $$

Denote 

$$\tilde{\tilde{\gamma_i}} = \frac{4\tilde{\gamma}_i}{\alpha \beta n (\alpha \beta n + 2)}, \tilde{\tilde{\gamma}} = (\tilde{\tilde{\gamma_1}}, \tilde{\tilde{\gamma_2}})^T \text{ and } M = \begin{pmatrix}
\frac{\alpha \beta n}{4} \Vert \mu_1 \Vert_2^2 + \sigma_1^2(\frac{\alpha \beta n}{2} + 1) & \frac{\alpha \beta n}{4} \mu_1^T\mu_2 \\
\frac{\alpha \beta n}{4} \mu_1^T\mu_2 & \frac{\alpha \beta n}{4} \Vert \mu_2 \Vert_2^2 + \sigma_2^2(\frac{\alpha \beta n}{2} + 1) \end{pmatrix}$$

The linear system of equations in $\tilde{\gamma}_1$ and $\tilde{\gamma}_2$ then translates as 

$$ M \tilde{\tilde{\gamma}} = \begin{pmatrix} -1 \\ 1 \end{pmatrix}$$

Note that $det(M) = \Delta$, where $\Delta$ is defined in the statement of the theorem. Hence, we deduce:

$$ \tilde{\tilde{\gamma}} = \frac{1}{\Delta} \begin{pmatrix}
\frac{\alpha \beta n}{4} \Vert \mu_2 \Vert_2^2 + \sigma_1^2(\frac{\alpha \beta n}{2} + 1) & -\frac{\alpha \beta n}{4} \mu_1^T\mu_2 \\
-\frac{\alpha \beta n}{4} \mu_1^T\mu_2 & \frac{\alpha \beta n}{4} \Vert \mu_1 \Vert_2^2 + \sigma_2^2(\frac{\alpha \beta n}{2} + 1) \end{pmatrix} \begin{pmatrix} -1 \\ 1 \end{pmatrix} $$

Which gives us 

$$ \tilde{\tilde{\gamma_1}} = -\frac{1}{\Delta}(\frac{\alpha\beta n}{4}(\Vert \mu_2 \Vert_2^2 + \mu_1^T\mu_2) + \sigma_1^2(\frac{\alpha \beta n}{2} + 1))$$
$$ \tilde{\tilde{\gamma_2}} = \frac{1}{\Delta}(\frac{\alpha\beta n}{4}(\Vert \mu_1 \Vert_2^2 + \mu_1^T\mu_2) + \sigma_2^2(\frac{\alpha \beta n}{2} + 1))$$

Recover $\tilde{\gamma_1}, \tilde{\gamma_2}$: 

$$ \tilde{\gamma_1} = -\frac{\alpha\beta n (\alpha \beta n + 2)}{4\Delta}(\frac{\alpha\beta n}{4}(\Vert \mu_2 \Vert_2^2 + \mu_1^T\mu_2) + \sigma_1^2(\frac{\alpha \beta n}{2} + 1))$$
$$ \tilde{\gamma_2} = \frac{\alpha\beta n (\alpha \beta n + 2)}{4\Delta}(\frac{\alpha\beta n}{4}(\Vert \mu_1 \Vert_2^2 + \mu_1^T\mu_2) + \sigma_2^2(\frac{\alpha \beta n}{2} + 1))$$

We can find the optimal $w$ using $\tilde{\gamma_1}$ and $\tilde{\gamma_2}$: 
$$w = - \frac{\alpha g + \tilde{\gamma_1} \mu_{1} + \tilde{\gamma_2} \mu_{2}}{2 + \alpha \beta n} = $$ 
$$ = -\frac{\alpha g}{2 + \alpha \beta n} + \frac{\alpha\beta n}{4\Delta}(\frac{\alpha\beta n}{4}(\Vert \mu_2 \Vert_2^2 + \mu_1^T\mu_2) + \sigma_1^2(\frac{\alpha \beta n}{2} + 1))\mu_1 - \frac{\alpha\beta n}{4\Delta}(\frac{\alpha\beta n}{4}(\Vert \mu_1 \Vert_2^2 + \mu_1^T\mu_2) + \sigma_2^2(\frac{\alpha \beta n}{2} + 1))\mu_2$$

Instead of directly inserting $\tilde{\gamma_1}$ and $\tilde{\gamma_2}$ into the objective now, which appears to be a horrendous task, we will remember what the optimal $w$ is but will return to the initial objective and change the order of optimization first using that the objective is convex in $w$ and concave in $\tilde{\gamma_1}, \tilde{\gamma_2}$:

$$\max_{\alpha \ge 0} \min_{\beta \ge 0} \min_{w} \max_{\tilde{\gamma_1}, \tilde{\gamma_2}}  \Vert w \Vert^2_2 + \alpha g^Tw + \frac{\alpha}{2\beta} + \frac{\alpha \beta n}{2}\Vert w \Vert_2^2 + \tilde{\gamma_1}(\mu_1^Tw-1) + \tilde{\gamma_2}(\mu_2^Tw+1)  - \frac{\sigma_1^2\tilde{\gamma_1} ^ 2 + \sigma_2^2{\gamma_2} ^ 2}{\alpha \beta n}$$

Differentiating by $\tilde{\gamma_1}$ and $\tilde{\gamma_2}$ and equating to $0$ again we immediately see that $\tilde{\gamma_1} = \frac{\alpha \beta n}{2\sigma_1^2}(\mu_1^Tw-1)$ and $\tilde{\gamma_2} = \frac{\alpha \beta n}{2\sigma_2^2}(\mu_2^Tw + 1)$. Incorporating this remark into the objective we get:

$$\max_{\alpha \ge 0} \min_{\beta \ge 0} \min_{w} \Vert w \Vert^2_2 + \alpha g^Tw + \frac{\alpha}{2\beta} + \frac{\alpha \beta n}{2}\Vert w \Vert_2^2 + \frac{\alpha \beta n}{4\sigma_1^2} (\mu_1^Tw-1)^2 + \frac{\alpha \beta n}{4\sigma_2^2}(\mu_2^Tw+1)^2 $$

Note that this is a quadratic function in $w$ whose linear term is equal to 
$$(\alpha g + \frac{\alpha\beta n}{2}(\frac{\mu_2}{\sigma^2_2} - \frac{\mu_1}{\sigma^2_1}))^Tw$$. 

Hence, the value of the objective at the optimal parameter $w$ equals $(\frac{\alpha}{2}g  + \frac{\alpha\beta n}{4}(\frac{\mu_2}{\sigma^2_2} - \frac{\mu_1}{\sigma^2_1}))^Tw$. Use the expression for the optimal $w$ we derived earlier to evaluate it at the optimal $w$. We will deem the cross-terms negligible because for a random standard Gaussian $g$ its dot products with $\mu_1$ and $\mu_2$ are negligible and will also replace $g^Tg$ by $d$:

$$(\frac{\alpha}{2}g  + \frac{\alpha\beta n}{4}(\frac{\mu_2}{\sigma^2_2} - \frac{\mu_1}{\sigma^2_1}))^Tw = -\frac{\alpha^2d}{4 + 2 \alpha \beta n} + \frac{(\alpha\beta n)^2}{16\Delta}(\frac{\alpha\beta n}{4}(\Vert \mu_2 \Vert_2^2 + \mu_1^T\mu_2) + \sigma_1^2(\frac{\alpha \beta n}{2} + 1))(\frac{\mu_1^T\mu_2}{\sigma_2^2} - \frac{\Vert \mu_1 \Vert_2^2}{\sigma_1^2}) - $$ 
$$- \frac{(\alpha\beta n)^2}{16\Delta}(\frac{\alpha\beta n}{4}(\Vert \mu_1 \Vert_2^2 + \mu_1^T\mu_2) + \sigma_2^2(\frac{\alpha \beta n}{2} + 1))(\frac{\Vert \mu_2 \Vert_2^2}{\sigma_2^2} - \frac{\mu_1^T\mu_2}{\sigma_1^2}) = $$
$$ = -\frac{\alpha^2d}{4(1 + \frac{\alpha\beta n}{2})} - \frac{(\alpha\beta n )^2}{16\Delta}(\Vert \mu_1 \Vert ^2 + \Vert \mu_2 \Vert ^ 2 - (\frac{\sigma_1^2}{\sigma_2^2} + \frac{\sigma_2^2}{\sigma_1^2})\mu_2^T\mu_1) - $$ 
$$ - \frac{(\alpha \beta n)^3}{32\Delta}(\Vert \mu_1 \Vert ^2 + \Vert \mu_2 \Vert ^ 2 - (\frac{\sigma_1^2}{\sigma_2^2} + \frac{\sigma_2^2}{\sigma_1^2})\mu_2^T\mu_1 + (\frac{1}{2\sigma_1^2} + \frac{1}{2\sigma_2^2})(\Vert \mu_1 \Vert^2 \Vert \mu_2 \Vert ^ 2 - (\mu_2^T\mu_1)^2))$$
\end{proof}

\begin{proof}{\bf of Theorem \ref{1SMD_CGMT}}

Plug in $\psi(w) = \Vert w \Vert_1$ in the objective of Theorem $\ref{psi_SMD_CGMT}$:

$$\max_{\alpha \ge 0} \min_{\beta \ge 0} \max_{\gamma_1, \gamma_2} \min_{w} \Vert w \Vert_1 + \alpha g^Tw + \frac{\alpha}{2\beta} + \frac{\alpha \beta n}{2}\Vert w \Vert_2^2 + \frac{\gamma_1(\mu_1^Tw-1)}{\sigma_1} + \frac{\gamma_2(\mu_2^Tw+1)}{\sigma_2} - \frac{\gamma_1 ^ 2 + \gamma_2 ^ 2}{\alpha \beta n}$$

Note that this expression can be split in $i$:

$$\max_{\alpha \ge 0} \min_{\beta \ge 0} \max_{\gamma_1, \gamma_2} \min_{w} \frac{\gamma_2}{\sigma_2} - \frac{\gamma_1}{\sigma_1}  + \frac{\alpha}{2\beta} - \frac{\gamma_1 ^ 2 + \gamma_2 ^ 2}{\alpha \beta n} + \sum_i |w_i| + \alpha g_iw_i + \frac{\alpha \beta n}{2} w_i^2  + \frac{\gamma_1}{\sigma_1}\mu_{1i}w_i + \frac{\gamma_2}{\sigma_2}\mu_{2i}w_i $$

Therefore, minimizing the entire expression in $w$ is equivalent to minimizing the corresponding summand for each $i$:

$$ \min_{w_i}  |w_i| + \alpha g_iw_i + \frac{\alpha \beta n}{2} w_i^2  + \frac{\gamma_1}{\sigma_1}\mu_{1i}w_i + \frac{\gamma_2}{\sigma_2}\mu_{2i}w_i $$

Denoting $u_i = |w_i|, \epsilon_i = sign(w_i)$ we rewrite it as:

$$\min_{u_i \ge 0, \epsilon_i = \pm 1}  u_i + u_i\epsilon_i(\alpha g_i  + \frac{\gamma_1}{\sigma_1}\mu_{1i} + \frac{\gamma_2}{\sigma_2}\mu_{2i}) + \frac{\alpha \beta n}{2} u_i^2$$

It is clear now that this is minimized when $\epsilon_i = - sign(\alpha g_i  + \frac{\gamma_1}{\sigma_1}\mu_{1i} + \frac{\gamma_2}{\sigma_2}\mu_{2i})$ and the problem reduces to:

$$\min_{u_i \ge 0}  u_i(1 - |\alpha g_i  + \frac{\gamma_1}{\sigma_1}\mu_{1i} + \frac{\gamma_2}{\sigma_2}\mu_{2i}|) + \frac{\alpha \beta n}{2} u_i^2$$

The latter is just a quadratic problem with a constraint $u_i \ge 0$ and therefore the solution is 
$u_i = \max(0, \frac{|\alpha g_i  + \frac{\gamma_1}{\sigma_1}\mu_{1i} + \frac{\gamma_2}{\sigma_2}\mu_{2i}| - 1}{\alpha \beta n})$. The corresponding value of the objective  is:
 $$ - \frac{|\alpha g_i  + \frac{\gamma_1}{\sigma_1}\mu_{1i} + \frac{\gamma_2}{\sigma_2}\mu_{2i}| - 1}{2\alpha \beta n} \max(0, |\alpha g_i  + \frac{\gamma_1}{\sigma_1}\mu_{1i} + \frac{\gamma_2}{\sigma_2}\mu_{2i}| - 1)$$
Recover the corresponding $w_i$:
$$w_i = \epsilon_i u_i = - sign(\alpha g_i  + \frac{\gamma_1}{\sigma_1}\mu_{1i} + \frac{\gamma_2}{\sigma_2}\mu_{2i})\max(0, \frac{|\alpha g_i  + \frac{\gamma_1}{\sigma_1}\mu_{1i} + \frac{\gamma_2}{\sigma_2}\mu_{2i}| - 1}{\alpha \beta n})$$

Replacing each $|w_i| + \alpha g_iw_i + \frac{\alpha \beta n}{2} w_i^2  + \frac{\gamma_1}{\sigma_1}\mu_{1i}w_i + \frac{\gamma_2}{\sigma_2}\mu_{2i}w_i$ in the main objective from the beginning of the proof with the value of the objective at the optimal $w_i$ we just found, we derive the desired four - dimensional optimization problem: 

$$\max_{\alpha \ge 0} \min_{\beta \ge 0} \max_{\gamma_1, \gamma_2} \frac{\gamma_2}{\sigma_2} - \frac{\gamma_1}{\sigma_1}  + \frac{\alpha}{2\beta} - \frac{\gamma_1 ^ 2 + \gamma_2 ^ 2}{\alpha \beta n} - \sum_i \frac{|\alpha g_i  + \frac{\gamma_1}{\sigma_1}\mu_{1i} + \frac{\gamma_2}{\sigma_2}\mu_{2i}| - 1}{2\alpha \beta n} \max(0, |\alpha g_i  + \frac{\gamma_1}{\sigma_1}\mu_{1i} + \frac{\gamma_2}{\sigma_2}\mu_{2i}| - 1) $$

\end{proof}

\end{document}